\documentclass[11pt]{article}

\usepackage[utf8]{inputenc} 
\usepackage[T1]{fontenc}    
\usepackage{hyperref}       
\usepackage{url}            
\usepackage{booktabs}       
\usepackage{amsfonts}       
\usepackage{nicefrac}       
\usepackage{microtype}      

\usepackage{microtype}
\usepackage{graphicx}
\graphicspath{{../Figures/}}
\usepackage{booktabs} 
\usepackage{caption}
\usepackage{subcaption}
\usepackage{amsmath}
\usepackage{amsfonts}
\usepackage{amssymb}
\usepackage{authblk}
\usepackage{cite}

\usepackage{amsopn,amssymb,amsthm,amsmath,bm}
\usepackage{paralist}
\usepackage{multirow}
\usepackage{multicol}
\usepackage{float}
\usepackage{thmtools}
\usepackage{thm-restate}
\usepackage{framed}
\usepackage{hyperref}
\usepackage{enumitem}

\usepackage[utf8]{inputenc} 
\usepackage[T1]{fontenc}    
\usepackage{hyperref}       
\usepackage{url}            
\usepackage{booktabs}       
\usepackage{amsfonts}       
\usepackage{nicefrac}       
\usepackage{microtype}      
\usepackage{lipsum}

\usepackage{algorithm}
\usepackage{algorithmic}
\usepackage{amsfonts}
\usepackage{amsmath}
\usepackage{amsthm}
\usepackage{amssymb}
\usepackage{epsf}
\usepackage{amssymb,amsmath}
\usepackage{mathrsfs}
\usepackage{bbm}
\usepackage{colordvi}
\usepackage{amscd}
\usepackage{epsfig}
\usepackage{balance}
\usepackage{mathtools}
\usepackage{wrapfig}

\DeclareMathOperator*{\argmin}{arg\,min}

\newcommand{\R}{\mathbb{R}}
\newcommand{\E}{\mathbb{E}}
\newcommand{\bb}{\pmb \beta}
\newcommand{\bs}{\pmb \beta^s}
\newcommand{\T}{\pmb \theta}
\newcommand{\bT}{\pmb \Theta}

\newcommand{\bbss}{ {\bs}^{\ast}}

\newcommand{\I}{\mathbf I}
\newcommand{\X}{\mathbf X}
\newcommand{\x}{\mathbf x}
\newcommand{\g}{\mathbf g}
\newcommand{\y}{\mathbf y}
\renewcommand{\S}{\mathbf S}
\newcommand{\B}{\mathbf B}
\newcommand{\s}{\mathbf s}
\renewcommand{\r}{\mathbf r}
\newcommand{\z}{\mathbf z}
\newcommand{\Z}{\mathbf Z}
\newcommand{\q}{\mathbf q}
\renewcommand{\H}{\mathbf H}
\newcommand{\bears}{\textsc{BEAR}}
\newcommand{\nn}{\nonumber}

\newtheorem{theorem}{Theorem}
\newtheorem{corollary}{Corollary}[theorem]
\newtheorem{lemma}[theorem]{Lemma}

\title{BEAR: Sketching BFGS Algorithm for Ultra-High Dimensional Feature Selection in Sublinear Memory}

\author{Amirali Aghazadeh}
\author{Vipul Gupta}
\author{Alex DeWeese}
\author{O. Ozan Koyluoglu}
\author{Kannan Ramchandran}

\affil{Department of Electrical Engineering and Computer Science \\ University of California, Berkeley}
\date{\vspace{-5ex}}

\begin{document}

\maketitle

\begin{abstract}%
We consider feature selection for applications in machine learning where the dimensionality of the data is so large that it exceeds the working memory of the (local) computing machine. Unfortunately, current large-scale sketching algorithms show poor memory-accuracy trade-off in selecting features in high dimensions due to the irreversible collision and accumulation of the stochastic gradient noise in the sketched domain. Here, we develop a second-order feature selection algorithm, called \bears{}, which avoids the extra collisions by efficiently storing the \emph{second-order stochastic gradients} of the celebrated Broyden–Fletcher–Goldfarb–Shannon (BFGS) algorithm in Count Sketch, using a memory cost that grows sublinearly with the size of the feature vector. \bears{} reveals an unexplored advantage of second-order optimization for memory-constrained high-dimensional gradient sketching. Our extensive experiments on several real-world data sets from genomics to language processing demonstrate that \bears{} requires \emph{up to three orders of magnitude} less memory space to achieve the same classification accuracy compared to the first-order sketching algorithms with a comparable run time. Our theoretical analysis further proves the global convergence of \bears{} with $\mathcal{O}(1/t)$ rate in $t$ iterations of the sketched algorithm.
\end{abstract}

\section{Introduction}
Consider a data set comprising $n$ data points $(\T_i)_{i=1}^n = (\x_i,y_i)_{i=1}^n$, where $\x_i\in\mathbb{R}^p$ denotes the data vectors representing $p$ features and $(y_i)_{i=1}^n$ denote the corresponding labels. Feature selection seeks to select a small subset of the features of size $k \ll p$ that best models the relationship between $\x_i$ and $y_i$. In this paper, we consider the feature selection problem in ultra-high dimensional settings where dense feature vectors in $\mathbb{R}^p$ cannot fit in the working memory of the computer because of the sheer dimensionality of the problem ($p$). Such problems have become increasingly important in biology, chemistry, networking, and streaming applications. In biology, it is common to represent a DNA sequence comprised of four nucleotides A, T, C, G, as well as $11$ wild-card characters in the FASTQ format~\cite{deorowicz2011compression}, using the frequency of sub-sequences of length $k$, called $k$-mers, with $k\geq 12$ \cite{vervier2016large,aghazadeh2016universal}. A feature vector of size $15^{k=12}$ with floating-point numbers requires more than a petabyte of memory to store. This is simply larger than the memory capacity of the computers today. In streaming, the memory budget of the local edge computing devices is extremely small compared to the dimension of the data streams \cite{yu2013software}. In both scenarios, it is critical to select a subset of the features that are most predictive of the outputs with \emph{lowest memory} cost in the dimensionality of the data.

Recently, first-order stochastic gradient descent (SGD) algorithms~\cite{aghazadeh2018mission,tai2018sketching} have been developed which extend the ideas in feature hashing (FH)~\cite{weinberger2009feature} to feature selection. Instead of explicitly storing the feature vectors, these algorithms 
store a low-dimensional sketch of the features in a data structure called Count Sketch~\cite{charikar2002finding}, originated from the streaming literature. The key idea behind Count Sketch is in linearly binning (colliding) a random subset of features into the same bucket (i.e., an entry of Count Sketch). As long as the total number of features with high weight (i.e., the heavy hitters) is small, such collisions won't affect the weights of the heavy hitters. In particular, Count Sketch preserves the weights of the top-$k$ features with high probability using a memory cost that grows sublinearly with the size of the feature vector (p). This high probability guarantee, however, depends on the energy of the non-top-$k$ coordinates in the SGD algorithm. In particular, the noise components of the gradients, which normally average out in the regular stochastic optimization, accumulate in the non-top-$k$ coordinates of Count Sketch. This unwanted sketched noise increases the probability of deleterious collision of heavy hitters in Count Sketch, deteriorates the quality of the recovered features, and results in poor memory-accuracy trade-offs. This is a critical problem since the only class of optimization algorithms that operates in such ultra-high dimensions does not select high-quality features when the memory budget is small. 

We propose a novel optimization scheme to solve this critical problem in sketching. We improve the quality of the sketched gradients and correct for the unwanted collisions in the sketched domain using the information from the second-derivative of the loss function. Second-order methods have recently gained increasing attention in machine learning for their faster convergence~\cite{agarwal2017second}, less reliance on the step size parameter~\cite{xu2020second}, and their superior communication-computation trade-off in distributed processing~\cite{yao2019pyhessian}. Here, we uncover another key advantage of second-order optimization in improving memory-accuracy trade-off in sketching models trained on ultra-high dimensional data sets. We develop a second-order optimization algorithm with a memory cost that grows sublinearly with the size of the feature vector (p). Our algorithm finds high-quality features by \emph{limiting the probability of extra collisions} due to the stochastic noise in Count Sketch. The contributions of the paper are as follows:

{\bf Algorithm.} We develop \bears{} which, to the best of our knowledge, is the first quasi-Newton-type algorithm that achieves a memory cost that grows sublinearly with the size of the feature vector (p). We demonstrate that applying sublinear memory data structures such as Count Sketch in the second-order optimization is challenging particularly because Hessian, unlike the gradient, cannot be directly sketched. In response, \bears{} stores the product of the inverse Hessian and the gradient in the Broyden–Fletcher Goldfarb–Shannon (BFGS) algorithm using Count Sketch. \bears{} updates Count Sketch in time quadratic in the sparsity of the data by operating only on the features that are active in each minibatch.

{\bf Theory.} We theoretically demonstrate that \bears{} maintains the $\mathcal{O}(1/t)$ global convergence rate of the online version of limited-memory BFGS algorithm (oLBFGS)~\cite{mokhtari2015global} in $t$ iterations. We show that the convergence rate is retained as we go from the ambient domain to the sketched domain. The analysis employs the matrix Bernstein inequality to bound the non-zero eigenvalues of the projection operator in Count Sketch. In practice, we demonstrate that \bears{} converges faster than the first-order feature selection algorithms---although improving convergence time is not the main focus of this work.

{\bf Experiments.} In real-world, ultra-high dimensional data sets from genomics, natural language processing, and networking, we demonstrate that \bears{} requires $10-1000\times$ less memory space to achieve the same classification accuracy as the first-order methods. Moreover, \bears{} achieves $10-20\%$ higher classification accuracy given the same memory budget to store Count Sketch and selects more interpretable features in ultra-high dimensions using a personal laptop-size machine. Importantly, our results show an increase in the performance gap between the first- and second-order methods as the memory budget to store the model parameters decrease, which highlights the important advantages of second-order optimization in storing the sketched stochastic gradient vectors with a lower collision rate.\footnote{Codes are available at https://github.com/BEAR-algorithm/BEAR}

 {\bf Simulations.} We performed extensive controlled sparse recovery simulations with data points drawn from the normal distribution. We demonstrate that, given a fixed memory budget to store the weights, \bears{} recovers the ground truth features with a large phase transition gap --- an important statistical performance metric from the compressive sensing literature. We show that \bears{}'s performance is highly consistent across a large range of values for the step size parameter because of the second-order nature of the algorithm.\\

\section{Review: Count Sketch}
Count Sketch is a data structure which is originated from the streaming literature~\cite{charikar2002finding}. Its primary application is to approximately count the number of occurrences of a very large number ($p$) of objects in sublinear memory when only the frequency of the most recurring elements (i.e., the heavy hitters) are of interest. Instead of storing a counter for all the $p$ objects, Count Sketch \emph{linearly} projects the count values using $d$ independent random hash functions into a $m \ll p$ dimensional subspace. Count Sketch keeps a matrix of counters (or bins) $\mathcal{S}$ of size $c\times d= m \sim \mathcal{O}(\log{p})$ and uses $d$ random hash functions $h_j: \forall j \in \{1,\ 2,...,\ d\}$ to map $p$-dimensional vectors to $c=m/d$ bins, that is, $h_j:\{1,\ 2,..., \ p\} \rightarrow \{1, \ 2,...,\  c\}$. For any row $j$ of sketch $\mathcal{S}$, component $i$ of the vector is hashed into bin $\mathcal{S}(j, h_j(i))$. In addition to $h_j$, Count Sketch uses $d$ random sign functions to map the components of the vectors randomly to $\{+1,\ -1\}$, that is, $s_j:\{1,\ 2,..., \ p\} \rightarrow \{+1,-1\}$. 

Count Sketch supports two operations: ADD(item i, increment $\Delta$) and QUERY(item i). The ADD operation updates the sketch with any observed increment. More formally, for an increment $\Delta$ to an item $i$, the sketch is updated by adding $s_j(i)\Delta$ to the cell $\mathcal{S}(j,h_j(i))$ $\forall  j \in \{1,\ 2,...,\ d\}$. The QUERY operation returns an estimate for component $i$, the median of all the $d$ different associated counters. In this paper, the objects that we aim to ``count'' are the weights of the coordinates of the gradient vector in the feature selection algorithm. Count Sketch provides the following bound in recovering the top-$k$ coordinates of the feature vector $\z \in \mathbb{R}^p$:

\begin{theorem}\label{thm:count_sketch} 
\cite{charikar2002finding} Count Sketch finds approximate top-$k$ coordinates $z_i$ with $\pm \varepsilon \|\z\|_2$ error, with probability at least $1-\delta$, in space $\mathcal{O}(\text{log}(\frac{p}{\delta}) (k + \frac{\|\z^{tail}\|_2^2}{ (\varepsilon \zeta)^2 }) )$, where $\|\z^{tail}\|_2^2 = \sum_{i \not\in top-k} z_i^2$ is the energy of the non-top-$k$ coordinates and $\zeta$ is the $k^\text{th}$ largest value in $\z$.
\end{theorem}

Count Sketch recovers the top-$k$ coordinates with a memory cost that grows only logarithmic with the dimension of the data $p$; and naturally, it requires the energy of the non-top-$k$ coordinates to be sufficiently small. This is the property that we leverage in this paper in order to improve feature selection accuracy in ultra-high dimensions.

\section{Stochastic Sketching for Feature Selection}
We first elaborate on how can we perform feature selection using Count Sketch. Recall that feature selection seeks to select a small subset of the features that best models the relationship between $\x_i$ and $y_i$. This relationship is captured using a sparse feature vector $\bb^{\ast} \in \mathbb{R}^p$ that minimizes a given loss function $f(\bb,\T):\mathbb{R}^p \rightarrow \mathbb{R}$ using the optimization problem $\min_{\bb} \E_{\T}{[f(\bb,\T)]}$,
where $\T \in \{ \T_1, \T_2, \cdots, \T_n \} $ denotes a data point in a data set of size $n$. This problem is solved using the empirical risk minimization
\begin{align} \label{eq:subset}
\bb^{\ast} &\coloneqq \argmin_{\bb} {\sum_{i=1}^n{f(\bb,\T_i)}},
\end{align}
using the SGD algorithm which produces the updates
$\bb_{t+1} \coloneqq \bb_{t} + \eta_t \g(\bb_{t},\bT_t)$
at iteration $t$, where $\eta_t$ is the step size, the minibatch $\bT_t = \{ \T_{t1},\T_{t2},\dots,\T_{tb}\}$ contains $b$ independent samples from the data, and $\g(\bb_{t},\bT_t) = \sum_{i=1}^b \nabla_{\bb_{t}} f(\bb_t,{\T}_{ti})$ is the stochastic gradient of the instantaneous loss function $f(\bb,{\bT}_{t})$. In this paper, we are interested in the setting where the dense feature vector $\bb_t$ of size $p$ cannot be stored in the memory of a computer. The most common approach in machine learning when dealing with such high dimensional problem is to project the data points (i.e., the features) into a lower dimensional space. Feature hashing (FH) is one of the most popular algorithms ~\cite{weinberger2009feature} which uses a universal hash function to project the features. While FH is ideal for prediction, it is not suited for feature selection; that is, the original important features cannot be recovered from the hashed ones. 

The reason to stay hopeful in recovering the important features using sublinear memory is that the feature vector $\bb^{\ast}$ is typically sparse in ultra-high dimensions. However, while the final feature vector is sparse, $\bb_t$ becomes dense in the intermediate iterations of the algorithm. The workaround is to store a sketch of the intermediate non-sparse $\bb_t$ using a low dimensional sketched vector $\bb_t^s\in\mathbb{R}^m$ (with $m\ll p$) such that the important features are still recoverable. This results in the following sketched optimization steps $\bb_{t+1}^s \coloneqq \bb_{t}^s + \eta_t \g^s(\bb_t,\bT_t)$, where $\g^s(\bb_t,\bT_t)$ is the sketched gradient vector. To enable the recovery of the important features from the sketched features the weights $\bb_{t}^s \in \mathbb{R}^m$ can be stored in Count Sketch~\cite{charikar2002finding,aghazadeh2018mission}. Count Sketch preserves the information of the top-$k$ elements (i.e., the heavy hitters) with high probability as long as the energy of the non-top-$k$ coefficients is sufficiently small (see Theorem~\ref{thm:count_sketch}) . The noise term $\g^s(\bb_t,\bT_t)$ in the SGD algorithm, however, contributes to the energy of the non-top-$k$ coefficients. This is a critical problem since, unlike SGD in the ambient dimension, this spurious sketched noise does not cancel out until it becomes so large that it shows up in the top-$k$ coordinates in Count Sketch. As a result, a large fraction of the memory in Count Sketch will be wasted to store the sketched noise term in the non-top-$k$ coordinates, which results in poor memory-accuracy trade off in selecting features in first-order methods. \\

\begin{wrapfigure}{R}{0.5\textwidth}
\begin{minipage}{0.5\textwidth}
\vspace{-0.8cm}
\begin{algorithm}[H]
   \caption{Limited-memory BFGS}
   \label{alg:algorithm_LBFGS}
\begin{algorithmic}
   \STATE {\bf Input:} $\g(\hat{\bb}_t,\bT_t)$ and $\{ \s_i,\r_i \}_{i={t-\tau+1}}^{t}$
   \STATE 1. $\rho_t = \frac{1}{\r_t^T\s_t}$.
   \STATE 2. $\,\,\q_{t}=\g(\hat{\bb}_t,\bT_t)$,
   \STATE $\,\,\;\;$ for $i= t$ to $t-\tau+1$: 
   \STATE $\;\;\;\;\;\;\;\;\;\;\alpha_i = \rho_i \s_i^T \q_i $,
   \STATE $\;\;\;\;\;\;\;\;\;\;\q_{i-1} = \q_{i} - \alpha_i \r_i$.
   \STATE 3. $\,\, \z_{t-\tau}=\frac{\r_t^T\s_t}{\r_t^T\r_t} \q_{t-\tau}$,
   \STATE $\,\,\;\;$ for $i= t-\tau+1$ to $t$:
   \STATE $\;\;\;\;\;\;\;\;\;\;\gamma_i = \rho_i \r_i^T \z_i $.
   \STATE $\;\;\;\;\;\;\;\;\;\; \z_{i} = \z_{i-1} + \s_i(\alpha_i - \gamma_i)$.
   \STATE {\bf Return:} $\z_t$
\end{algorithmic}
\end{algorithm}
\vspace{-0.6cm}
    \end{minipage}
  \end{wrapfigure}
\section{Challenges of Second-order Sketching in Ultra-High Dimension}
In this paper, we propose a second-order optimization algorithm for feature selection to reduce the effect of collisions while sketching SGD into lower dimensions. Recall that the stochastic second-order Newton's method produces the updates
$\bb_{t+1} \coloneqq \bb_{t} + \eta_t \B_t^{-1} \g(\bb_t,\bT_t)$,
where $\B_t =  \nabla_{\bb_t}^2f(\bb_t,\bT_t) \in \mathbb{R}^{p\times p}$ is the instantaneous Hessian at iteration $t$ computed over the minibatch $\bT_t$. However, there are critical challenges in sketching these updates for ultra-high dimensional feature selection: First, computing the Hessian and finding the matrix inverse is computationally hard as the matrix inverse operation is going to have cubic computational complexity in the problem dimension (i.e., $\mathcal{O}(p^3)$) in the worst case. Recent works have allowed for efficient inversion of the Hessian matrix with a computational complexity which grows linearly with $p$ (i.e., $\mathcal{O}(pr^2)$) using a rank-$r$ approximation of Hessian or grouping the eigenvalues of Hessian into $r$ clusters (see e.g., ~\cite{bollapragada2019exact,o2019inexact}). However, the computational complexity of all these efficient algorithms grow in the best case linearly with $p$. Second, even if assuming that the Hessian is approximately diagonal (e.g., in AdaHessian~\cite{yao2020adahessian}), storing the diagonal elements will have linear memory cost in $p$, which we can not afford in ultra-high dimension. Third, sketching the Hessian directly is not possible using Count Sketch since the linear increments are happening over the gradients and not the Hessian in SGD.

Unfortunately, all the recent literature in fast second-order optimization also operates with at least linear memory and time complexity. Namely, Quasi-Newton's methods such as the BFGS algorithm reduce the time complexity of Newton's method using an iterative update of the Hessian as a function of the variations in the gradients $\r_t = \g(\bb_{t+1},\bT) - \g(\bb_{t},\bT)$ and the feature vectors $\s_t = \bb_{t+1} - \bb_{t}$ which ensures that the Hessian satisfies the so-called secant equation $\B_{t+1} \s_t  = \r_t$.
While BFGS avoids the heavy computational cost involved in the matrix inversion, it still has a quadratic memory requirement. The limited-memory BFGS (LBFGS) algorithm reduces the memory requirement of the BFGS algorithm from quadratic to linear by estimating the product of the inverse Hessian and the gradient vector $\z_t = \B_t^{-1} \g(\bb_t,\bT_t)$ without explicitly storing the Hessian~\cite{nocedal1980updating}. This makes use of the difference vectors $\r_t$ and $\s_t$ from the last $\tau$ iteration of the algorithm (Alg.~\ref{alg:algorithm_LBFGS}). Recently, an online LBFGS algorithm (oLBFGS)~\cite{mokhtari2015global} with global linear convergence guarantee and an incremental greedy BFGS algorithm (IGS)~\cite{gaoincremental} with non-asymptotic local superlinear convergence guarantee have been developed. However, both the online limited-memory and incremental BFGS algorithms fail to run on ultra-high dimensional data sets due to their linear memory requirement. This motivates the central question that we aim to address: how can we enjoy the benefits of second-order optimization for feature selection in sublinear memory?

\section{The \bears{} Algorithm}
Our feature selection algorithm, \bears{}, estimates the second-derivative of the loss function from sketched features. Instead of explicitly storing the product of the inverse Hessian and the gradient (done in oLBFGS), \bears{} maintains a Count Sketch $\mathcal{S}$ to store the feature weights in sublinear memory and time quadratic in the sparsity of the input data. The key insight is to update the Count Sketch by the sketch of the gradient corrected by the sketch of the difference vectors $\r_t$ and $\s_t$.
\begin{algorithm}[t]
   \caption{\bears{}}
   \label{alg:algorithm}
\begin{algorithmic}
   \STATE {\bf Initialize}: $t = 0$, Count Sketch $\bb_{t=0}^s = 0$, top-$k$ heap.
   \WHILE {stopping criteria not satisfied}
   \STATE 1. Sample $b$ independent data points in a minibatch  $\bT_t = \{ \T_{t1},\dots,\T_{tb} \}$. 
   \STATE 2. Find the active set $\mathcal{A}_t$. 
   \STATE 3. QUERY the feature weights in $\mathcal{A}_t \cap \text{top-}k$ from  Count Sketch $\bb_t = query(\bb_t^s )$.
   \STATE 4. Compute stochastic gradient $\g(\bb_t,\bT_t)$.
   \STATE 5. Compute the descent direction with Alg. \ref{alg:algorithm_LBFGS} 
   $\z_t$ = LBFGS($\g(\bb_t,\bT_t)$ , $\{ \s_i ,\r_i \}_{i={t-\tau+1}}^{t}$).
   \STATE 6. ADD the sketch of $\z_t$ at the active set $\hat{\z}_t = {\z_t}^{\mathcal{A}_t}$  to Count Sketch 
   $\bb_{t+1}^s \coloneqq \bb_{t}^s - \eta_t {\hat{\z}_t}^s $.
   \STATE 7. QUERY the features weights in $\mathcal{A}_t \cap \text{top-}k$ from   Count Sketch $\bb_{t+1} = query(\bb_{t+1}^s )$.
   \STATE 8. Compute stochastic gradient $\g(\bb_{t+1},\bT_t)$.
   \STATE 9. Set $\s_{t+1} = \bb_{t+1} - \bb_t$, and $\r_{t+1} =  \g(\bb_{t+1},\bT_t) - \g(\bb_t,\bT_t)$.
   \STATE 10. Update the top-$k$ heap.
   \STATE 11. $t = t + 1$.
   \ENDWHILE
   \STATE {\bf Return:} The top-$k$ heavy-hitters in Count Sketch. 
\end{algorithmic}
\end{algorithm}
As detailed in Alg.~\ref{alg:algorithm}, \bears{} first initializes Count Sketch with zero weights and a top-$k$ heap to store the top-$k$ features. In every iteration, it samples $b$ independent data points $\bT_t$ and identifies the active set $\mathcal{A}_t$, that is, the features that are present in $\bT_t$. It then queries Count Sketch to retrieve the feature weights that are in the intersection of the active set $\mathcal{A}_t$ and the top-$k$ heap and set the weights for the rest of the features to zero. Next, it computes the stochastic gradient $\g(\bb_t,\bT_t)$ and uses it along with the difference vectors $\r_i$ and $\s_i$ from the last $\tau$ iterations to find the descent direction $\z_t$ using the LBFGS algorithm detailed in Alg.~\ref{alg:algorithm_LBFGS}. Then, it adds the sketch of the descent direction $\z_t$ \emph{only at the features in the active set} $\hat{\z}_t = {\z_t}^{\mathcal{A}_t}$ to Count Sketch. \bears{} queries Count Sketch for the second time in order to update the difference vector $\r_{t+1} =  \g(\bb_{t+1},\bT_t) - \g(\bb_t,\bT_t)$ and uses the sketch vector $\hat{\z}_t$ to set $\s_{t+1}$. The difference vector $\r_{t+1}$ captures the changes in the gradient vector as the content of Count Sketch change \emph{over a fixed minibatch} $\bT_t$~\cite{mokhtari2015global}. Finally, \bears{} updates the top-$k$ heap with the active set $\mathcal{A}_t$ and moves on to the next iteration until the convergence criteria is met. To update the top-$k$ heap, \bears{} scans the features that have been changed in Count Sketch over the past iteration. If those features are already in the heap, it updates the values of those elements, and if the features are new, it inserts the new elements into the heap with a worst-case time complexity that grows logarithmic with the number of features $k$. In the rare scenario, where the intersection of the active set $\mathcal{A}_t$ and the top-$k$ heap is empty the gradient can still be non-zero and can change the feature weightings.

The most time-costly step of \bears{}, which computes the descent direction (step 5), is quadratic time in the sparsity of the data $| \mathcal{A}_t|$. Table~\ref{table:vectors} summarizes the worst-case memory complexity of the vectors involved in \bears{}. The dominant term is Count Sketch $\bb_t^s$ for which the memory complexity in terms of the top-$k$ value are established in Theorem~\ref{thm:count_sketch}. The memory requirement to store the auxiliary vector $\z_t$ is only a constant $\tau$ times larger than the size of the active set which is negligible in streaming data-sparse features compared to the size of Count Sketch (see section \ref{experiments}). \\

\noindent {\bf Convergence Analysis.} We now analyze the convergence of the algorithm. Consider the following standard assumptions \cite{mokhtari2015global} on the instantaneous functions $f(\bb, \bT)$ to prove the convergence of the \bears{} algorithm: 1) The instantaneous objective functions are twice differentiable with the instantaneous Hessian being positive definite, that is, the eigenvalues of the instantaneous Hessian satisfy $M_1\I \preceq \nabla^2_{\bb} f(\bb, \bT) \preceq M_2\I$, for some $0 < M_1 < M_2$. 2) The norm of the gradient of the instantaneous functions is bounded for all $\bb$, that is, $\E_{\bT}[\|  \g(\bb, \bT) \|^2~|~ \bb] \leq S^2$. 3) The step sizes $\eta_t$ are square-summable. More specifically, $\sum_{t=0}^{\infty} \eta_t = \infty \text{ and } \sum_{t=0}^{\infty} \eta_t^2 < \infty$. We prove the following theorem.

\begin{table}[t]
\begin{center}
    \begin{tabular}{ |c|c|c|c|c|c| } 
    \hline
    $\bb_t$ & $\s_t$ & $\r_t$ & $\z_t$ & $\bb_t^s$ & $\g(\bb_t,\bT_t)$   \\
    \hline
    $k$ & $2|\mathcal{A}_t|$ & $2|\mathcal{A}_t|$ & $2\tau |\mathcal{A}_t| $ & $|\mathcal{S}|$ & $|\mathcal{A}_t|$  \\
    \hline
    \end{tabular}
\end{center}
\vspace{-0.3cm}
\caption{Memory cost of the vectors in \bears{}.}
\label{table:vectors}
\end{table}
  
\begin{theorem}\label{thm:prob_conv}
Let $f(\cdot)$ and the step sizes $\eta_t$ satisfy the assumptions above. 
Let the size of Count Sketch be $m = \theta(\varepsilon^{-2}\log 1/\delta)$ with number of hashes $d = \theta(\varepsilon^{-1}\log 1/\delta)$ for $\varepsilon,\delta>0$.
Then, the Euclidean distance between updates $\bb_t^s$ in the \bears{} algorithm and the sketch of the solution of problem \eqref{eq:subset} converges to zero with probability $1-\delta$, that is, 
\begin{align}\label{eq:convergence}
\mathbb{P}(\lim_{t\rightarrow\infty}\|\bb_t^s - {\bs}^*\|^2 = 0) = 1 - \delta,
\end{align}
where the probability is over the random realizations of random samples $\{\bT_t\}_{t=0}^{\infty}$.
Furthermore, for the specific step size $\eta_t = \eta_0/(t+T_0)$ for some constants $\eta_0$ and $T_0$, the model parameters at iteration $t$ satisfy 
\begin{equation}
\E_{\bT}[f(\bs_t,\bT) - \E[ f({\bs}^\ast,\bT)] \leq \frac{C_0}{T_0+t},
\end{equation}
with probability $1-\delta$. Here, $C_0$ is a constant depending on the parameters of the sketching scheme, the above assumptions, and the objective function.
\end{theorem}

The proof makes use of the matrix Bernstein inequality in projecting the second-order gradients in Count Sketch~\cite{kane2014sparser} which we defer to the Proofs section. For sufficiently sparse solutions the convergence in the ambient domain follows from convergence in the sketched domain (i.e., Theorem~\eqref{thm:prob_conv}) and the Count Sketch guarantee (i.e., Theorem~\eqref{thm:count_sketch}):

\begin{corollary}
Let $\pi(\cdot)$ be a permutation on $\{1,2,\cdots,p\}$ such that $\beta^*_{\pi(1)}\geq \beta^*_{\pi(2)}\geq \cdots \beta^*_{\pi(p)}$, where $\bb^\ast = [\beta^*_1,\beta^*_2, \cdots ,\beta^*_p]$ is the optimal solution to \eqref{eq:subset}. Also, let
\begin{align}
m = \max\left[\mathcal{O}\left(\text{log}(\frac{2p}{\delta}) \left(k + \frac{\|{\bb^\ast}^{tail}\|_2^2}{ (\varepsilon \zeta)^2 }\right) \right), \theta\left(\frac{1}{\varepsilon^{2}}\log \frac{2}{\delta}\right)\right]
\end{align}

\noindent and number of hashes $d = \theta(\varepsilon^{-1}\log 2/\delta)$ where $\varepsilon,\delta>0, \|{\bb^\ast}^{tail}\|_2^2 = \sum_{i =k+1}^p (\beta^\ast_{\pi(i)})^2 \textit{ and } \zeta = \beta^\ast_{\pi(k)}$. Then,
\begin{align}
|\beta^\ast_{\pi(i)} - {\beta_{t\pi(i)}}| \leq \epsilon \|\bb^\ast\|_2 ~~~ \text{  for all } i\in \{1,2,\cdots, k\} \text{ with probability } 1 - \delta,
\end{align}

\noindent where ${\bb_t} = [{\beta_{t1}},{\beta_{t2}}, \cdots ,{\beta_{tp}}]$ is the output of the \bears{} algorithm.
\end{corollary}

\noindent This completes the convergence proof of \bears{} in sublinear memory in the ambient space.

\vspace{-0.2cm}
\section{Simulations}
\vspace{-0.2cm}
We have conducted sparse recovery simulations to evaluate the performance \bears{} compared to the first-order feature selection algorithm in ultra-high dimension MISSION~\cite{aghazadeh2018mission}. MISSION is one of the only first-order optimization algorithms that enables feature selection in a memory cost that grows sublinearly with the size of the feature vector and as a result runs in the scale of problems that we are considering in this paper. In addition, comparing the accuracy of BEAR with MISSION as a baseline shows the power of sketching the second-order gradients (done in BEAR) against sketching first-order gradients (done in MISSION). The synthetic simulations described in this section have ground truth features, so we can assess the algorithms in a more controlled environment and compare the results using a variant of the phase transition plot from the compressive sensing literature~\cite{maleki2010optimally}. We also show the results of the full Newton's method version of our \bears{} algorithm where we compute the Hessian rather than its oLBFGS approximation (this algorithm cannot operate in large-scale settings). 
The same hash table (hash functions and random seeds) and step sizes are used for \bears{} and MISSION. Hyperparameter search is performed to select the value of the step sizes in both algorithms.
The entries of the data vectors $\x_i$ are sampled from an i.i.d. Gaussian distribution with zero mean and unit variance. The output labels $y_i$ are set using a linear forward model $y_i = \x_i \bb^\ast$, where $\bb^\ast$ is a $k$-sparse ground truth feature vector. The indices of the support (i.e., the non-zero entries) and the weights of the non-zero entries in $\bb^\ast$ are drawn uniformly at random respectively from the sets $[1,p]$ and $[0.8,1.2]$ and MSE is used as the loss function. The same experiment is repeated $200$ times with different realization of the data vectors $\x$. Convergence at iteration $t$ is reached when the norm of the gradient drops below $10^{-7}$ consistently in all the algorithms. The algorithms are compared in terms of the accuracy in selecting the ground truth features as well as the sensitivity of the algorithms to the choice of the value of the step size. 

{\bf Feature Selection Accuracy.} The task is to select $k = 8$ features in a data set with $n = 900$ rows (data points) and $p = 1000$ columns (features). The size of Count Sketch is varied from $10\%$ to $60\%$ of the total memory required to store a $p=1000$ dimensional feature vector. This ratio, that is the ratio of data dimension $p$ to Count Sketch size, is called the compression factor. For each value of the compression factor, the experiment is repeated $200$ times. Fig.~\ref{fig:sweep_hash_size}A shows the fraction of iterations in which the algorithms find \emph{all} the ground truth features correctly, that is, the probability of success. 
Fig.~\ref{fig:sweep_hash_size}B illustrates the same results in terms of the average $\ell_2$-norm of the error of the recovered feature vectors $\| {\bb}_t-\bb^\ast \|_2$. \bears{} significantly outperforms MISSION in terms of the probability of success and the average $\ell_2$-norm error. The gap is more pronounced in higher compression factors; given a compression factor of $3$, MISSION has almost no power in predicting the correct features while the \bears{} and Newton's methods achieve a $0.5$ probability of success. Fig.~\ref{fig:sweep_hash_size}A and B further suggest that the performance gap between \bears{} and it's exact Hessian counterpart is small showing that the oLBFGS makes a good approximation to the Hessian in terms of the selected features.

\begin{figure*}[t]
\vspace{-0.7cm}
\centering
\includegraphics[width=1\textwidth]{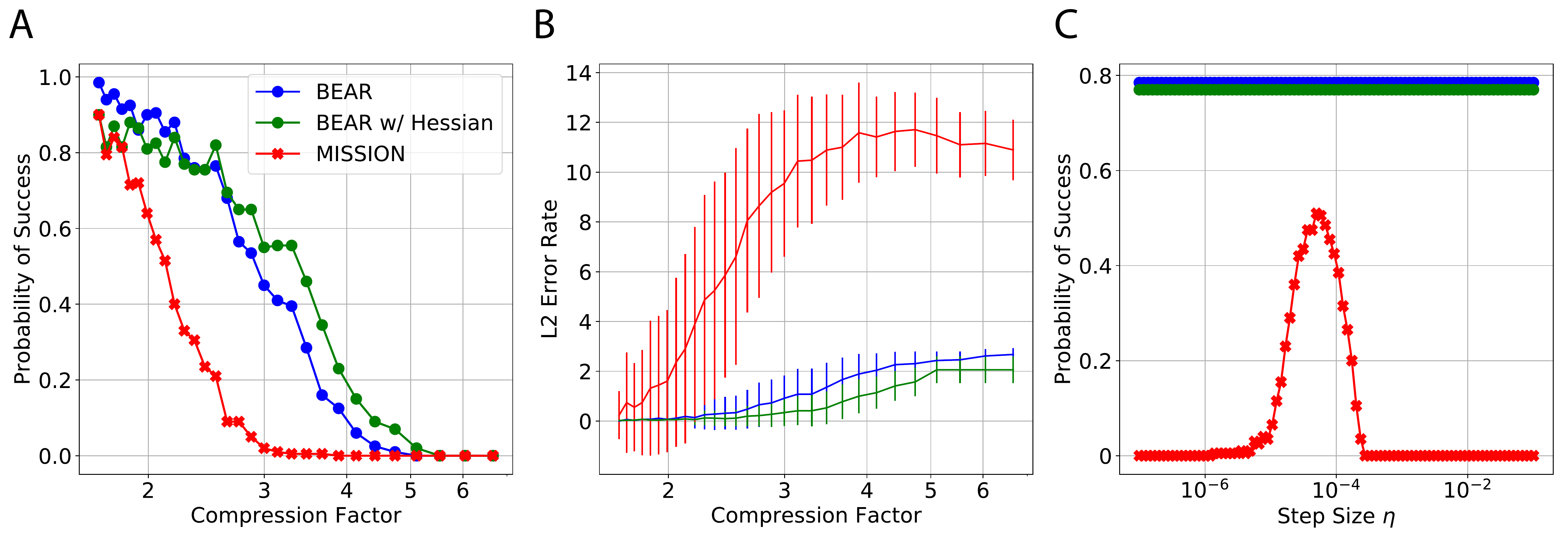}
\vspace{-0.6cm}
\caption{Feature selection experiments on $p=1000$-dimensional synthetic data sets with entries drawn from normal distribution. A) Probability of success in recovering all features correctly as a function of compression factor. B) Recovery error rate in terms of the $\ell_2$-norm. C) Probability of success as a function of the value of the step size (compression factor $ = 2.22$).}
\label{fig:sweep_hash_size}
\vspace{-0.6cm}
\end{figure*}

{\bf Sensitivity to Step Size.} The experimental setup is similar to the previous section except the Sketch size is fixed and step size varies. The experiment is repeated $200$ times while varying the values of the step size $\eta$ ranging from $10^{-7}$ to $10^{-1}$ and the probability of success is reported. Count Sketch of size $150\times 3$ is used for both MISSION and \bears{}. Fig.~\ref{fig:sweep_hash_size}C illustrates the probability of success for \bears{} and MISSION as a function of the step size. The plot shows that \bears{} is fairly agnostic and MISSION is dependent on the choice of the step size $\eta$. MISSION's accuracy peaks around $\eta=10^{-4}$ and sharply drops as $\eta$ deviates from this value. \bears{}'s lower-dependence on step size is ideal for streaming settings where the statistics of the data might change over time and there is not enough time and memory budget to do step size selection.

\vspace{-0.3cm}
\section{Experiments}\label{experiments}
\vspace{-0.2cm}
We designed the experiment in a way to answer the following questions:
\begin{itemize}
\item
Does \bears{} outperform MISSION in terms of classification accuracy? In particular, how does the performance gap between the algorithms change as a function of the memory allocated for Count Sketch? 
\item
How does \bears{} perform on real-world large-scale data sets ($p$ > 50 million)?
\item
How does \bears{} perform in terms of classification accuracy compared to FH? 
\item
How does changing the number of top-$k$ features affect the accuracy of the feature selection algorithms? 
\item
What is the convergence behaviour of \bears{} when the memory budget is small? 
\item
What is the run time of \bears{} compared to MISSION?
\end{itemize}

We compare the performance of these baseline algorithms with \bears{}: 1) {\bf Stochastic Gradient Descent (SGD)}: For data sets with sufficiently small dimension and size to be able to train a classifier on our laptop machine, we perform the vanilla SGD algorithm (with $\mathcal{O}(p)$ memory). 2) {\bf oLBFGS}: Similar to SGD, for the data sets that the dimension and size allows to train a classifier on our laptop machine, we perform the vanilla oLBFGS algorithm (neither SGD nor the oLBFGS techniques do feature selection or model compression). 3) {\bf Feature Hashing (FH)}: FH~\cite{weinberger2009feature} is a standard algorithm to do prediction (classification) in large-scale machine learning problems. FH hashes the data features into a lower dimensional space \emph{before} the training process and is not a feature selection algorithm. 4) {\bf MISSION}: As mentioned earlier, MISSION is a first-order optimization algorithm for feature selection which sketches the noisy stochastic gradients into Count Sketch. \\

\noindent {\bf Performance Metrics.} The algorithms are assessed in terms of the following performance metrics: 
1)~{\bf Classification accuracy}: Once the algorithms converge, the performance of the algorithms in terms of classification accuracy are compared, that is, the fraction of test samples that are classified to correct classes.
2)~{\bf Area under the ROC curve (AUC)}: For the data sets that the class distribution are highly skewed the area under the ROC curve (AUC) is reported instead of the classification accuracy. In these data sets, the class probabilities are taken as the output of the classifiers. 
3)~{\bf Compression factor (CF)}: The compression factor is defined as the dimension of the data set $p$ divided by the size of Count Sketch $m$. For multi-class classification problems, $m$ is the total memory of all the Count Sketches used for all the classes. A higher compression factor means a smaller memory budget is allocated to store the model parameters. SGD and oLBFGS have a compression factor of one. 4)~{\bf Run time}: The run time of the algorithms to converge in minutes. \\

\noindent {\bf Real-World Data sets.} The key statistics of the data sets used in the paper are tabulated in Table~\ref{table:datasets} including the dimension of the data set ($p$), number of training data ($n$), number of test data, total size of the data set, and the average number of active (non-zero) features per data point. All the data is analyzed in the Vowpal Wabbit format.

\begin{table} [t]
\small
\vspace{-0.3cm}
\begin{center}
    \begin{tabular}{ |c|c|c|c|c|c| } 
    \hline
    Data set & Dim ($p$) & \#Train ($n$) & \#Test & Size & \#Act. \\
    \hline
    RCV1 & 47,236 & 20,242 & 677,399 & 1.2GB & 73 \\
    Webspam & 16,609,143 & 280,000 & 70,000 & 25GB & 3730 \\
    DNA & 16,777,216 & 600,000 & 600,000 & 1.5GB & 89  \\
    KDD 2012 & 54,686,452 & 119,705,032 & 29,934,073 & 22GB & 12 \\
    \hline
    \end{tabular}
\end{center}
\vspace{-0.4cm}
\caption{Summary of the real-world data sets.}
\vspace{-0.3cm}
\label{table:datasets}
\end{table}

1) {\bf RCV1: Reuters Corpus Volume I.} RCV1 is an archive of manually categorized news wire stories made available by Reuters, Ltd. for research purposes. The negative class label includes Corporate/Industrial/Economics topics and positive class labels includes Government/Social/Markets topics (see \cite{lewis2004rcv1}). The data set is fairly balanced between the two classes. 

2) {\bf Webspam: Web Spam Classification.} Web spam refers to Web pages that are created to manipulate search engines and Web users. The data set is a large collection of annotated spam/nonspam hosts labeled by a group of volunteers (see \cite{webb2006introducing}). It is slightly class-imbalanced with $60\%$ samples from class 1. 

3) {\bf DNA: Metagenomics.} A data set that we dub ``DNA'' from metagenomics. Metagenomics studies the composition of microbial samples collected from various environments (for example human gut) by sequencing the DNA of the living organisms in the sample. The data set comprises of short DNA sequences which are sampled from a set of $15$ DNA sequences of bacterial genomes. The task is to train a classifier to label the DNA sequences with their corresponding bacteria. DNA sequences are encoded using their constituent sub-sequences called $K$-mers (see \cite{vervier2016large}). The training and test data have an equal number of samples for each class. A naive guessing strategy achieves a classification accuracy of $0.06$. 

4) {\bf KDD Cup 2012: Click-Through Rate Prediction.} A key idea in search advertising is to predict the click-through rate (pCTR) of ads, as the economic model behind search advertising requires pCTR values to rank ads and to price clicks. The KDD Cup 2012 data set comprises training instances derived from session logs of the Tencent proprietary search engine (see \cite{juan2016field}). The data set is highly class-imbalanced with $96\%$ samples from class 1 (click). \\

\noindent {\bf Multi-class Extension.} For the multi-class classification problems stated above, we developed a multi-class version of the BEAR algorithm. In the multi-class problem one natural assumption is that there are separate subsets of features that are most predictive for each class. Our multi-class BEAR algorithm accommodates for this by maintaining a separate Count Sketch and heap to store the the top-$k$ features associated with each class. The total memory complexity of the algorithm grows linearly with the number of classes. For a fair comparison, we use the exact same multi-class Count Sketch extension for MISSION. We have also implemented the single Count Sketch version of BEAR, however, since the multi-class Count Sketch extension performs better for feature selection we report the results of the former in our experiments.  \\

\noindent {\bf Experimental Setup.} MurmurHash3 with 32-bit hash values is used to implement the hash functions in MISSION, BEAR, and FH. The algorithms are trained in a streaming fashion using the cross entropy loss. The algorithms are run for a single epoch so that each algorithm sees a data point once on average. The size of the minibatches and the step size are kept consistent across the algorithms. The constant $\tau = 5$ in \bears{} however the results are consistent across  a large range of values for $\tau$. Both in \bears{} and MISSION a Count Sketch with 5 rows (hash functions) is used. The lower dimensional embedding size of FH is set equal to the total size of Count Sketch in \bears{}. The experiments are performed on a single laptop machine - 2.4 GHz Quad-Core Intel Core i5 with 16 GB of RAM. 
We chose an edge device as opposed to a computing server for our real-world experiments to showcase the applicability of \bears{} in a resource constrained environment. \\

\begin{figure*}[t]
\vspace{-0.8cm}
\centering
\includegraphics[width=1\textwidth]{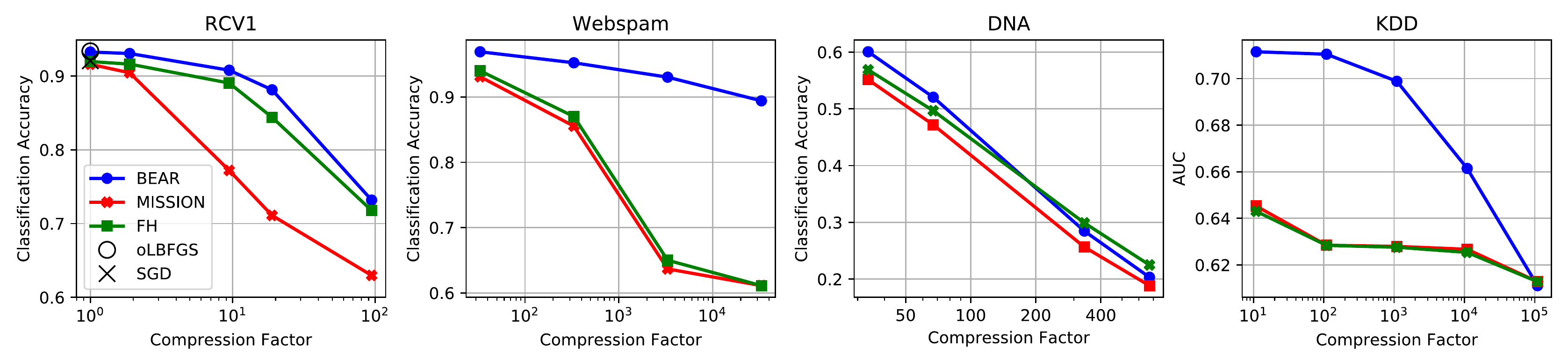}
\vspace{-0.7cm}
\caption{Classification performance as a function of the compression factor in real-world data sets.}
\vspace{-0cm}
\label{fig:accuracy-compression}
\end{figure*}

\begin{figure*}[t]
\vspace{-0.2cm}
\centering
\includegraphics[width=1\textwidth]{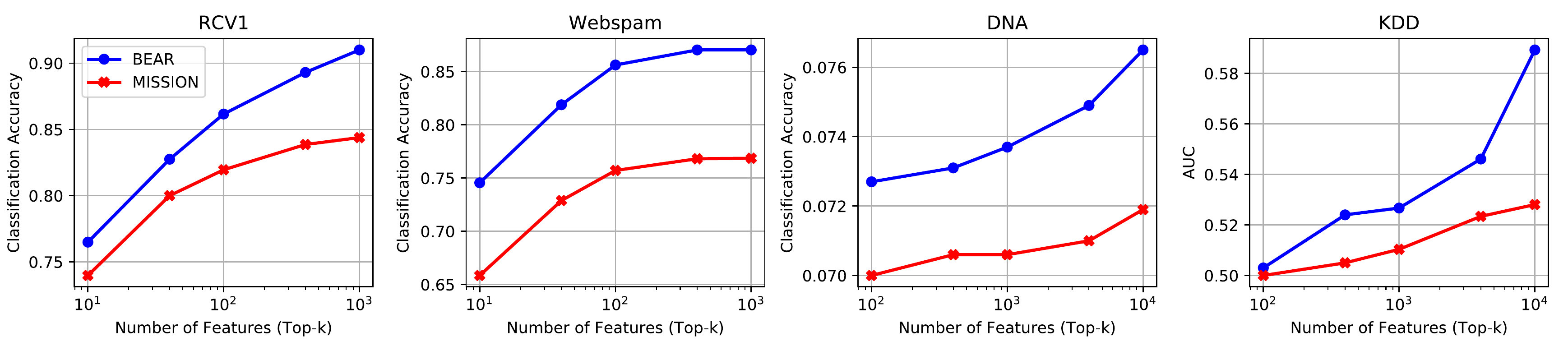}
\vspace{-0.7cm}
\caption{Classification performance as a function of the number of the top-$k$ features.}
\label{fig:topk}
\vspace{-0.4cm}
\end{figure*}

\noindent {\bf Result I) Classification Performance vs. Compression Factor.} We assess the classification performance of \bears{} compared to the baseline algorithms for different compression factors in Fig.~\ref{fig:accuracy-compression}. All the active features in the test data are used at the inference step for a fair comparison with FH. 
\bears{}'s classification performance is consistently better than MISSION and FH across all the data sets over a wide range of compression factors while showing a hysteresis behaviour: the performance gap increases as the compression factor grows until Count Sketch is too small to yield any prediction power.
The classification performance of all algorithms degrades with larger compression factors, which is expected since lower Count Sketch sizes increase the probability of collisions in both \bears{} and MISSION. The degradation, however, impacts MISSION significantly more that \bears{}. In particular, \bears{}'s performance stays relatively robust for compression rates in the range of $1-10$ in RCV1, $1-1000$ in Webspam, and $10-100$ in KDD, while the classification performance of MISSION drops rapidly. The increasing performance gap between \bears{} and MISSON with compression factor highlights the unique advantage of \bears{} in storing the second-order steps in Count Sketch and lowering the probability of collisions. Note that this performance gap is less pronounced in the DNA data set while the general trend still follows the other data sets. This is because the DNA data set has $15$ balanced classes and its $K$-mer features have relatively more distributed information content compared to the features in the other data sets, which poses a harder feature selection task for the algorithms. \\

\noindent {\bf Result II) Classification Performance vs. Top-$k$ Features.} We assess the performance of \bears{} in terms of the classification accuracy against the number of selected top-$k$ features. The compression factors are fixed to $10$, $330$, $330$, and $1100$ respectively for the data sets in Fig.~\ref{fig:topk}. SGD, oLBFGS, and FH cannot select features, therefore, they are not included in this analysis. The plots shows that \bears{} selects features that are better in terms of prediction accuracy for a wide range of values of $k$. The gap grows for larger $k$. We analyzed the selected features in RCV1 for which a proper documentation of the features is publicly available (unlike the other data sets). Some of the selected features are shared among the algorithms, for example, ``shareholder'', ``nigh'', and ``company'', which can be attributed to the Markets, Social, and Industrial subjects, respectively. Other terms, however, are uniquely chosen by one of the algorithms as tabulated in Table~\ref{table:interpretation}. Compared to \bears{}, the terms selected by MISSION are less frequent (e.g., ``peach'') and do not discriminate between the subject classes (e.g., ``incomplete''). \\

\begin{table} [t]
\vspace{-0.7cm}
\begin{center}
    \begin{tabular}{ |c|ccccc| } 
    \hline
    \bears{} & manage & entrepreneur & colombian & decade & oppress    \\
    \hline
    MISSION & peach & week & nora & demand & incomplete  \\
    \hline
    \end{tabular}
\end{center}
\vspace{-0.6cm}
\caption{Examples of the features selected in RCV1.}
\vspace{-0cm}
\label{table:interpretation}
\end{table}

\begin{table}[t]
\begin{center}
    \begin{tabular}{ |c|c|c|c|c| } 
    \hline
    data set (CF) & RCV1 ($95$)  & Webs ($332$) & DNA ($22$) & KDD ($10^3$)    \\
    \hline
    BEAR & $0.1$ & $5$ & $26$ & $25$  \\
    MISSION & $0.3$ & $19$ & $55$  & $33$  \\
    \hline
    \end{tabular}
\end{center}
\vspace{-0.6cm}
\caption{Overall run time comparison (minutes).}\label{table:runningtime}
\vspace{-0.2cm}
\end{table}

\noindent {\bf Result III) Run Time.} We compare the overall run time of \bears{} with MISSION in Table~\ref{table:runningtime}. \bears{} is significantly faster than MISSION consistently in all the data sets; \bears{} makes a better use of the data by estimating the curvature of the loss function and converges faster.

\sloppy

\vspace{-0.3cm}
\section{Discussion and Conclusion}
\vspace{-0.2cm}

We have developed \bears{}, which to the best of our knowledge is the first second-order optimization algorithm for ultra-high dimensional feature selection in sublinear memory.  Our results demonstrate that \bears{} has up to three orders of magnitude smaller memory footprint in feature selection compared to the first-order sketching algorithms with comparable (and sometimes superior) run time. We showed that the benefits of BEAR is far more pronounced while sketching into lower-dimensional subspaces, which is due to the more accurate decent directions of second-order gradients resulting in less collision-causing noise in Count Sketch. The implications of memory-accuracy advantage of second-order methods goes beyond hashing and streaming and can be applied to improve the communication-computation trade-off in distributed learning in communicating the sketch of the stochastic gradients between nodes~\cite{ivkin2019communication,gupta2019oversketched}. Moreover, while we laid out the algorithmic principles in sketching second-order gradient for training ultra-high dimensional generalized linear classifiers with theoretical guarantees, similar algorithmic principles can be used in sketching nonlinear models such as deep neural networks on lower-dimensional data sets~\cite{yao2020adahessian}. We believe that our work will open up new research directions towards understanding the benefits of second-order optimization in training massive-scale machine learning models in memory-constrained environments. \\


\noindent {\bf Proofs} \\

{\bf Theorem 2. \label{thm:norm_error}} Before stating the proof, for more clarity, we will reiterate the problem setup and our assumptions from the main paper here.
We are interested in solving the following problem using \bears{} 
\begin{equation}\label{orig_prob}
\bb^{\ast} \coloneqq \argmin_{\bb\in \R^p} f(\bb, \bT) =
\argmin_{\bb\in \R^p} \frac{1}{T}\sum_{t=1}^T{f(\bb,\T_t)}
= \argmin_{\bb\in \R^p} \frac{1}{T}\sum_{t=1}^T{f(\X_t\bb,\y_t)},  
\end{equation}
where $\bT = \{ \T_1, \T_2, \cdots, \T_T\}$ and $\T_t = (\X_t,\y_t)~\forall~t\in [1,T].$ We make the following standard assumptions \cite{mokhtari2015global}:
\begin{enumerate}
    \item 
The instantaneous objective functions, $f(\cdot)$, in Eq.~\eqref{orig_prob} are twice differentiable with the instantaneous Hessian being positive definite. That is, the eigenvalues of the instantaneous Hessian satisfy
    \begin{align}\label{assump1}
        M_1\I \preceq \nabla^2_{\bb}  f(\bb, \bT) \preceq M_2\I,
    \end{align}
    for some $0 < M_1 \leq M_2$.
\item The norm of the gradient of the instantaneous functions $f(\cdot)$ in Eq. \ref{orig_prob} is bounded for all $\bb$, that is
\begin{align}\label{assump2}
    \E_{\bT}[\|  \g(\bb, \bT) \|^2~|~ \bb] \leq S^2.
\end{align}
\item The step-sizes $\eta_t$ are square-summable. More specifically,
\begin{equation}\label{ss_cond}
    \sum_{t=0}^{\infty} \eta_t = \infty \text{ and } \sum_{t=0}^{\infty} \eta_t^2 < \infty.
\end{equation}
\end{enumerate}

\begin{lemma}
The solution of problem in \eqref{orig_prob} using \bears{} (or its first-order variant MISSION) is equivalent to the solution of the following problem in the sketched domain,
\begin{equation}\label{comp_prob}
{\bs}^{\ast} \coloneqq \argmin_{\bs\in \R^m} \frac{1}{T}\sum_{t=1}^T {f(\X_t\S\bs,\y_t)},
\end{equation}
where multiplication by $\S \in \R^{p\times m}$ is the linear projection operator in Count Sketch and $\bs \in \R^m$ is the projected model parameters.
\end{lemma}
\begin{proof}
Let the update for online gradient descent  for the original problem in Eq. \eqref{orig_prob} be given by
\begin{align}
    \bb_{t+1} =\bb_t - \eta_t\nabla f(\X\bb_t, \y_t) 
\end{align}
For BEAR/MISSION type algorithms, the model parameters are stored in a Count Sketch based hash table. The compressed vector can be represented by an affine transformation as $\bs_t = \S^T\bb_t$, where $\S\in \R^{p\times m}$ is the Count Sketch matrix \cite{kane2014sparser}. While updating the model, the indices corresponding to the non-zero values in the gradient (the oLBFGS update in case of BEAR) are updated by querying Count Sketch. For Count Sketch with mean query operator, the update for MISSION can be written as 
\begin{align}
    \bs_{t+1} =\bs_t - \eta_t\S^T\nabla f(\X \mathcal Q(\bs_t),\y_t) 
\end{align}
where $\mathcal Q(\cdot)$ is the query function and $\bs_t = \S^T\bb_t$ is the sketched model parameter vector. When the query is the mean operator, the $\mathcal{Q}(\cdot)$ is the affine transformation $\mathcal{Q}(x) = \S\x$ for any $\x \in \R^m$ \cite{kane2014sparser,woodruff2014sketching}. Thus, the MISSION update equation is given by
\begin{align}\label{mission_update}
    \bs_{t+1} =\bs_t - \eta_t\S^T\nabla f(\X\S\bs_t,\y_t). 
\end{align}
The gradient for the problem in Eq. \eqref{comp_prob} is given by $\nabla f_{\bs}(\cdot) = \S^T\nabla f(\cdot)$. Hence, its online gradient descent update is the same as MISSION's update in Eq. \eqref{mission_update}. Since BEAR is a second-order variant of MISSION, it attempts to solve the same problem as MISSION. Next, we show that it indeed solves the problem with high probability at a linear convergence rate.  
\end{proof}

Now, to show that \bears{} converges to $\bbss$, we first need to show that the problem in Eq. \eqref{comp_prob} also satisfies the assumptions in Eq. \eqref{assump1} and \eqref{assump2} (albeit with different constants). Then, we can invoke the convergence guarantees for oLBFGS from \cite{mokhtari2015global} to show that \bears{} converges at a linear rate.

\begin{lemma}
Assume Count Sketch has a size $m=\Theta(\varepsilon^{-2}\log 1/\delta)$ with the number of hashes $d=\Theta(\varepsilon^{-1}\log1/\delta)$. If the instantaneous function $ f(\bb, \bT)$ satisfy Assumptions 1 and 2 (Eq. \eqref{assump1} and \eqref{assump2}, respectively), then, the corresponding instantaneous function for the sketched problem $ f(\S\bs, \bT)$ also satisfy
\begin{align}
       \frac{p}{m}(1 - \varepsilon)M_1\I \preceq \nabla^2_{\bs} f(\S\bs, \bT) \preceq \frac{p}{m}M_2(1 + \varepsilon)\I, \label{assump:eig_sketched}\\
       \E_{\T}[\| \nabla_{\bs} f(\S\bs, \bT) \|^2~|~ \bb] \leq \frac{p}{m}M_2(1 + \varepsilon) S^2,
    \end{align}
with probability $1-\delta$ each.
\end{lemma}
\begin{proof}
The instantaneous Hessian for the sketched problem (say $\H^s$) is given by
\begin{align}
 \H^s =  \nabla^2_{\bs} f(\S\bs, \bT) = \S^T\nabla^2 f(\S\bs, \bT)\S = \S^T\H\S,
\end{align}
where $\H = \nabla^2 f(\S\bs, \bT)$ is the instantaneous Hessian for the original problem.
Since commuting matrices have the same set of non-zero eigenvalues, the eigenvalues of $\S^T\H\S$ are equal to the eigenvalues of $\H\S\S^T$. 
Hence, 
\begin{align}
\lambda_{\max}(\H^s) &=  \lambda_{\max}(\S^T\H\S) = \lambda_{\max}(\H\S\S^T)
\\
& \leq \lambda_{\max}(\H) \lambda_{\max}(\S\S^T) \label{ineq1} \\
&\leq M_2~ \lambda_{\max}(\S\S^T) \\
&= M_2~ \lambda_{\max}(\S^T\S),
\end{align}
where $\lambda_{\max}(\cdot)$ denotes the maximum eigenvalue and $\lambda_{\max}(\H) \leq M_2$ by assumption. Also, Eq. \eqref{ineq1} uses the fact that the maximum eigenvalue of the product of two symmetric matrices is upper bounded by the product of maximum eigenvalues of individual matrices. 

For the count sketch matrix $\S\in \R^{p\times m}$, we have $\E[\S^T\S] = \frac{p}{m}\I$. Moreover, by applying the matrix Bernstein inequality \cite{tropp2015introduction} on the matrix $\mathbf Z = \S^T\S - \frac{p}{m}\I = \sum_{i=1}^p \left(\S_i^T\S_i - \frac{1}{m}\I\right) = \sum_{i=1}^p \mathbf Z_i$, where $\S_i$ is the $i$-th row in $\S$ and $ \mathbf Z_i = \left(\S_i^T\S_i - \frac{1}{m}\I\right) ~\forall~ i\in [1,p]$, we get the following bound
    $\mathbb P\left(\|\mathbf Z\| \geq \epsilon\frac{p}{m}\right) \leq 2m \exp{\frac{-\epsilon^2p^2/(2m^2)}{v(\mathbf Z) + L\epsilon p/(3m)}}\nn,$
where $\|\mathbf Z_i\| \leq L ~\forall~i$ and $v(\mathbf Z) = \|\sum_{i=1}^p \Z_i^T\Z_i\|$. For the count sketch matrix, we have $L = 1$ and $v(\Z) = \frac{d(k-1)}{k^2}$. Thus, we get
    $\mathbb P\left(\|\mathbf Z\| \geq \epsilon\frac{p}{m}\right) \leq 2m \exp{\frac{-\epsilon^2p^2/(2m^2)}{p(m-1)/m^2 + \epsilon p/(3m)}}\nn.$
Further, for any $m = O(\sqrt{p})$, the R.H.S. in the above inequality is upper bounded by $\delta$. Note that the sketch-size $m$ is generally independent and (order-wise) much less than the bigger dimension $p$, and hence $m = O(\sqrt{p})$ can be easily satisfied by choosing appropriate constants $\epsilon > 0$ and $\delta < 1$. Thus, we get the following bound on the eigenvalues of $\S^T\S$ (also the non-zero eigenvalues in $\S\S^T$)
\begin{align}\label{eigs_bound_sketch}
    \frac{p}{m}(1 -\varepsilon) \leq \lambda_i(\S^T\S) \leq \frac{p}{m}(1 +\varepsilon)
    \text{ for all } i \in [1,m]
\end{align}
with probability $1-\delta$. 
Using this in \eqref{ineq1}, we get
\begin{equation}
 \lambda_{\max}(\H^s)  \leq \frac{p}{m}M_2(1 + \varepsilon)
\end{equation}
with probability $1-\delta$.

Similarly, we can write the smallest eigenvalue of $\H^s$ as
\begin{align}
    \lambda_{\min}(\H^s) = \frac{1}{\lambda_{\max}((\H^s)^{-1})} = \frac{1}{\lambda_{\max}((\S^T\H\S)^{-1})} = \frac{1}{\lambda_{\max}((\H\S\S^T)^{\dagger})},
\end{align}
where $(\cdot)^{\dagger}$ is the Moore-Penrose inverse and the last inequality again uses the fact that commuting matrices have the same set of non-zero eigenvalues. Thus, $\S^T\H\S$ and $\H\S\S^T$, and their corresponding inverses, have the same set of non-zero eigenvalues.
Let's define the truncated eigenvalue decomposition of $\S\S^T$ as $\S\S^T = \mathbf U \Lambda \mathbf U^T$ and note that $(\H\S\S^T)^{\dagger} = (\H\mathbf U \Lambda \mathbf U^T)^{\dagger} = (\mathbf U^T)^{\dagger}\Lambda^{-1}(\mathbf U)^{\dagger}\H^{-1}$. Hence, we get 
\begin{align}
\lambda_{\max}((\H\S\S^T)^{\dagger}) &= \lambda_{\max}((\mathbf U^T)^{\dagger}\Lambda^{-1}(\mathbf U)^{\dagger}\H^{-1})\nn \\ 
&\leq \lambda_{\max}(\Lambda^{-1}) \lambda_{\max}((\H)^{-1})\nn\\
&= \lambda_{\max}((\S\S^T)^{\dagger}) \lambda_{\max}((\H)^{-1})
\end{align}
since $\Lambda^{-1}$ contains the non-zero eigenvalues of $(\S\S^T)^{\dagger}$.
Thus,
\begin{align}
    \lambda_{\min}(\H^s) &\geq \frac{1}{\lambda_{\max}((\S\S^T)^{\dagger}) \lambda_{\max}((\H)^{-1})}\nn \\
    &\geq \lambda_{\min}(\S\S^T)\lambda_{\min}(\H)\nn\\
    &\geq \lambda_{\min}(\S\S^T)M_1\nn \\
    &= \lambda_{\min}(\S^T\S)M_1 \nn\\
     &\geq \frac{p}{m}M_1(1 - \varepsilon),
\end{align}
with probability $1-\delta$. where the last inequality follows from \eqref{eigs_bound_sketch}. This proves the desired result.

Similarly, to prove that the gradient of the sketched problem is bounded, observe that $\| \nabla_{\bs}  f(\S\bs, \bT) \|^2 =\|\S^T \nabla  f(\S\bs, \bT)\|^2 \leq \frac{p}{m}M_2(1 + \varepsilon)\| \nabla f(\S\bs, \bT) \|^2$ with probability $1 - \delta$, where the last inequality follows from \eqref{eigs_bound_sketch}.
Hence, 
\begin{align}\label{assump:exp_sketched}
\E_{\T}[\| \nabla_{\bs}  f(\S\bs, \bT) \|^2 \leq \frac{p}{m}M_2(1 + \varepsilon)\E_{\T}[\| \nabla f(\S\bs, \bT) \|^2 \leq \frac{p}{m}M_2(1 + \varepsilon)S^2
\end{align}
with probability $1-\delta$, where the second inequality follows from assumption in \eqref{assump2}.
\end{proof}

Finally, to prove Theorem 2, we invoke the results from \cite{mokhtari2015global}. According to Theorem 6 in \cite{mokhtari2015global}, oLBGS with instantaneous functions satisfying assumptions in Eqs. \eqref{assump:eig_sketched} and \eqref{assump:exp_sketched} converges with probability one. Hence, for BEAR, we get
\begin{align}
    \mathbb{P}(\lim_{t\rightarrow\infty}\|\bb_t^s - {\bs}^*\|^2 = 0) = 1 - \delta.
\end{align}

Moreover, for the specific step-size $\eta_t = \eta_0/(t+T_0)$, where $\eta_0$ and $T_0$ satisfy the inequality $2m_1\eta_0T_0 > C$ for some constant $C$, BEAR satisfies the following rate of convergence (Theorem 7 in \cite{mokhtari2015global})
\begin{equation}
    \E[f(\bs_t,\bT) - \E[ f({\bs}^\ast,\bT)] \leq \frac{C_0}{T_0+t},
\end{equation}
with probability $1-\delta$,  
where the constant $C_0$ is given by
\begin{align*}
    C_0 = \max\left\{\frac{\eta_0^2T_0^2CM_2p^2S^2(1 + \varepsilon)^2}{2c^2m[M_1p\eta_0T_0(1 - \varepsilon) - Cm]}, T_0(\E[f(\bs,\bT) - \E[ f({\bs}^\ast,\bT)])\right\}.
\end{align*}

\noindent {\bf Acknowledgements} \\
This work is supported by the following grants: NSF CIF-1703678, NSF CNS-1748692, NSF CCF-1748585, and MLWiNS 2002821.

\bibliographystyle{ieeetr}
 \bibliography{BEAR-Arxiv.bib}

\end{document}